\documentclass[letterpaper]{article}
\usepackage[preprint]{aaai2027}
\usepackage[hyphens]{url}
\usepackage{graphicx}
\usepackage{natbib}
\usepackage{caption}
\usepackage{amsmath}
\usepackage{amssymb}
\usepackage{amsthm}
\usepackage{booktabs}
\newtheorem{proposition}{Proposition}
\newtheorem{theorem}{Theorem}

\title{R$^2$A: Learning Persona Policies Through Persona Representation Learning and Runtime Alignment}
\author{
Mohan Zhang\textsuperscript{\rm 1},
Chengsong You\textsuperscript{\rm 1},
Xiaoyu Cao\textsuperscript{\rm 1},
Zhen Sun\textsuperscript{\rm 1},\\
Xiaohan Jia\textsuperscript{\rm 1},
Junwei Zhou\textsuperscript{\rm 2},
Yongchao Chen\textsuperscript{\rm 1}\corresponding
}
\affiliations{
\textsuperscript{\rm 1}College of AI, Tsinghua University\\
\textsuperscript{\rm 2}Independent Researcher\\
zhangmohan70@gmail.com, 51275901122@stu.ecnu.edu.cn, xiaoyucorynne@gmail.com,\\
sunzhen@apexin.ai, jiaxh23@mails.tsinghua.edu.cn, zjw330501@gmail.com,\\
yongchaochen12@gmail.com
}

\begin{document}
\maketitle

\begin{abstract}
The same Persona behavior can be beneficial in one context but harmful in another, causing static Persona elicitation to perform inconsistently across tasks. We introduce the \textbf{Persona Selection--Realization Framework}, which models behavior generation through a latent Persona state and decomposes it into Persona Selection and Persona Realization. The discrepancies between static Persona elicitation and an ideal Persona policy in these two components define the Selection Gap and Realization Gap, respectively. Building on this framework, we propose \textbf{R$^2$A}, a two-stage approach for learning Persona policies. Persona Representation Learning uses structured Who--How--What presentations to encode the target Persona's objective, conditional behavioral principles, and trajectory-level manifestations. Persona Runtime Alignment then removes the explicit Persona specification and jointly calibrates behavior selection and trajectory realization using task feedback. Across 12 evaluation settings covering the four principles of the Accountable-Professional Persona studied in this work, R$^2$A overall outperforms both the base model and static Persona elicitation. Ablation results further show that Persona Representation Learning is critical for preventing Runtime Alignment from producing behaviorally imbalanced policies and for achieving more stable Persona policy learning.
\end{abstract}

\section{Introduction}

Persona specifications provide a lightweight natural-language interface for steering the behavior of large language models (LLMs): by describing an identity, objectives, and behavioral principles, users can alter how a model responds and interacts. A simple approach is to reuse the same Persona specification across different tasks, which we refer to as \textbf{static Persona elicitation}. Prior work reports mixed effects: role-play prompts can improve zero-shot reasoning \citep{kong2024roleplay}, system-prompt Personas often yield no consistent benefit \citep{zheng2024helpful}, and coarsely matched Personas can degrade performance \citep{kim2025persona}.

To study this phenomenon under controlled conditions, we introduce the \textbf{Accountable-Professional Persona}, which describes a professional agent that takes responsibility for both the problem-solving process and its eventual outcome. Its four behavioral principles and their complementary behaviors are shown in Figure~\ref{fig:persona}. We compare the base model with the same model under static Accountable-Professional elicitation on benchmarks aligned with these principles. As shown in Figure~\ref{fig:static-results}, the same Persona prompt produces improvements, limited changes, or degradations across different tasks. In other words, static elicitation can alter the model's behavioral tendencies, but cannot consistently translate a high-level Persona specification into a mode of task handling appropriate for the current context. This motivates our central question:

\begin{quote}
\textbf{Through what process does static Persona elicitation translate a high-level Persona specification into concrete task behavior, and why does the same elicitation perform inconsistently across tasks?}
\end{quote}

\begin{figure}[t]
\centering
\includegraphics[width=\columnwidth]{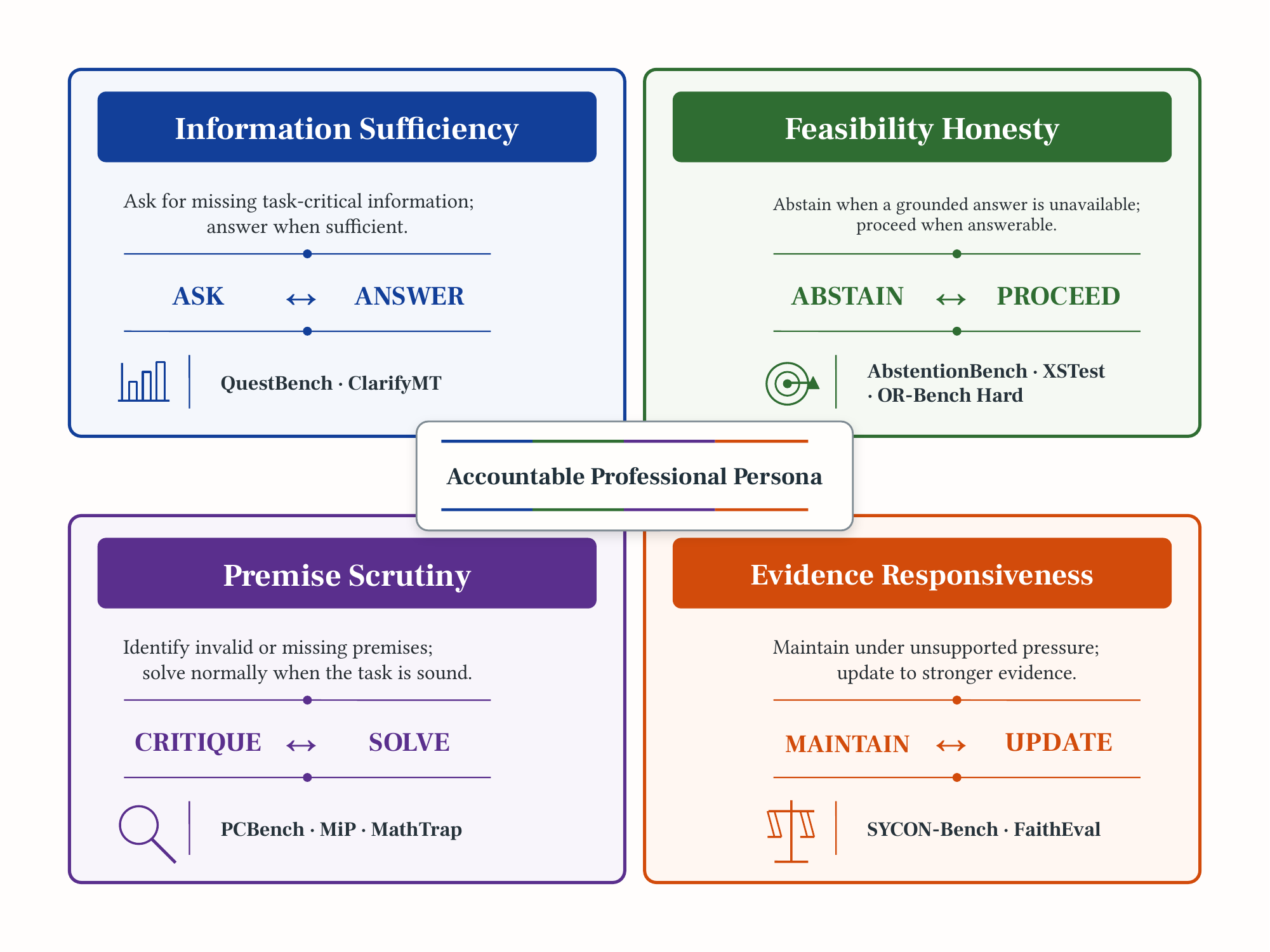}
\caption{The Accountable-Professional Persona and its four evaluated principles. Each pairs a signature behavior with its complementary expression and corresponding benchmark probes.}
\label{fig:persona}
\end{figure}

Inspired by the Persona Selection Model (PSM) \citep{marks2026psm}, we introduce the \textbf{Persona Selection--Realization Framework}. The framework uses a latent Persona state $z$ to represent the concrete instantiation of a high-level Persona in the current task context, and decomposes the model's inference process from task context to complete output trajectory into two components. \textbf{Persona Selection} characterizes how the model selects $z$ based on the task context and external specification, while \textbf{Persona Realization} characterizes how the model generates a complete output trajectory given $z$. Relative to an ideal Persona policy, we further decompose the deviation induced by static Persona elicitation into a \textbf{Selection Gap} and a \textbf{Realization Gap}, corresponding respectively to selecting an inappropriate Persona state and inadequately realizing the selected state.

To reduce the Selection--Realization Gap, we propose R$^2$A (Representation-to-Alignment). In the first stage, \textbf{Persona Representation Learning} uses structured Who--How--What presentations to teach the model the target Persona's overarching objective, conditional behavioral principles, and complete trajectory-level manifestations. In the second stage, \textbf{Persona Runtime Alignment} uses trajectory utility across different task contexts to calibrate Persona Selection and Persona Realization within the same policy. The former provides shared structure and applicability boundaries for Persona behaviors, while the latter adjusts how those behaviors are expressed according to task utility.

Our contributions are twofold. \textbf{First}, to explain why the same static Persona prompt performs inconsistently across tasks, we introduce the \textbf{Persona Selection--Realization Framework}. It uses a latent Persona state $z$ to characterize how a high-level Persona specification is translated into concrete task behavior, decomposes this process into Persona Selection and Persona Realization, and formalizes the deviation of static elicitation from an ideal Persona policy as a Selection Gap and a Realization Gap. \textbf{Second}, building on this framework, we propose R$^2$A, which establishes a structured representation of the target Persona through Persona Representation Learning and calibrates high-level behavior selection and trajectory realization through Persona Runtime Alignment. Our main experiments validate the effectiveness of R$^2$A, while the ablation studies further analyze its key design choices and its advantages over static Persona elicitation.

\begin{figure}[t]
\centering
\includegraphics[width=\columnwidth]{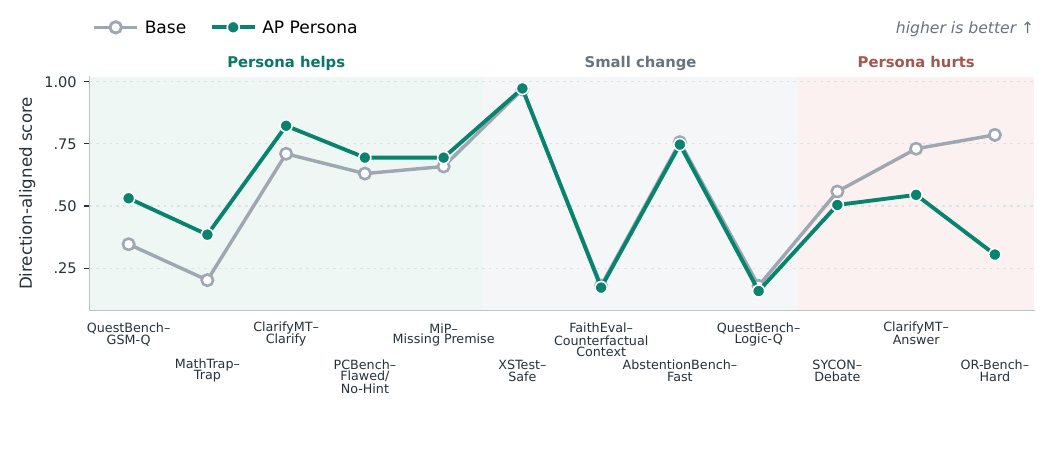}
\caption{Direction-aligned performance of the base model and the same model under static Accountable-Professional elicitation across evaluation tasks. Higher scores are better.}
\label{fig:static-results}
\end{figure}

\section{Persona Selection--Realization Framework}

\subsection{Accountable-Professional Persona and Static Elicitation}

\paragraph{Accountable-Professional Persona.}
We define an Accountable-Professional as a professional agent that takes responsibility for both the problem being addressed and its eventual outcome. Its behavioral principles comprise four dimensions.

\textbf{Information Sufficiency} requires the model to ask targeted clarification questions when task-critical information is missing and to answer directly when the available information is sufficient.

\textbf{Feasibility Honesty} requires the model to state its limitations when a task cannot be completed reliably and to proceed when the task is feasible.

\textbf{Premise Scrutiny} requires the model to identify invalid or missing premises and to solve the task normally when its premises are sound.

\textbf{Evidence Responsiveness} requires the model to maintain its position when supported by the available evidence and to update it when stronger evidence emerges.

The four principles and their complementary behaviors are illustrated in Figure~\ref{fig:persona}. We operationalize them using corresponding public benchmarks; detailed settings are provided in Section~\ref{sec:experiments} and the appendix.

\paragraph{Static Persona Elicitation.}
We compare the base model with the same model under static Accountable-Professional elicitation. The two settings differ only in whether a fixed Persona specification is included in the inference input; all other configurations are held constant.

As shown in Figure~\ref{fig:static-results}, the same Persona prompt improves performance on some tasks, has limited effects on others, and causes degradation in several cases. Because every task uses the same Persona specification, these results show that static elicitation can alter the model's behavioral tendencies but cannot consistently translate Persona principles into a mode of task handling appropriate for the current context.

\begin{figure*}[t]
\centering
\includegraphics[width=\textwidth]{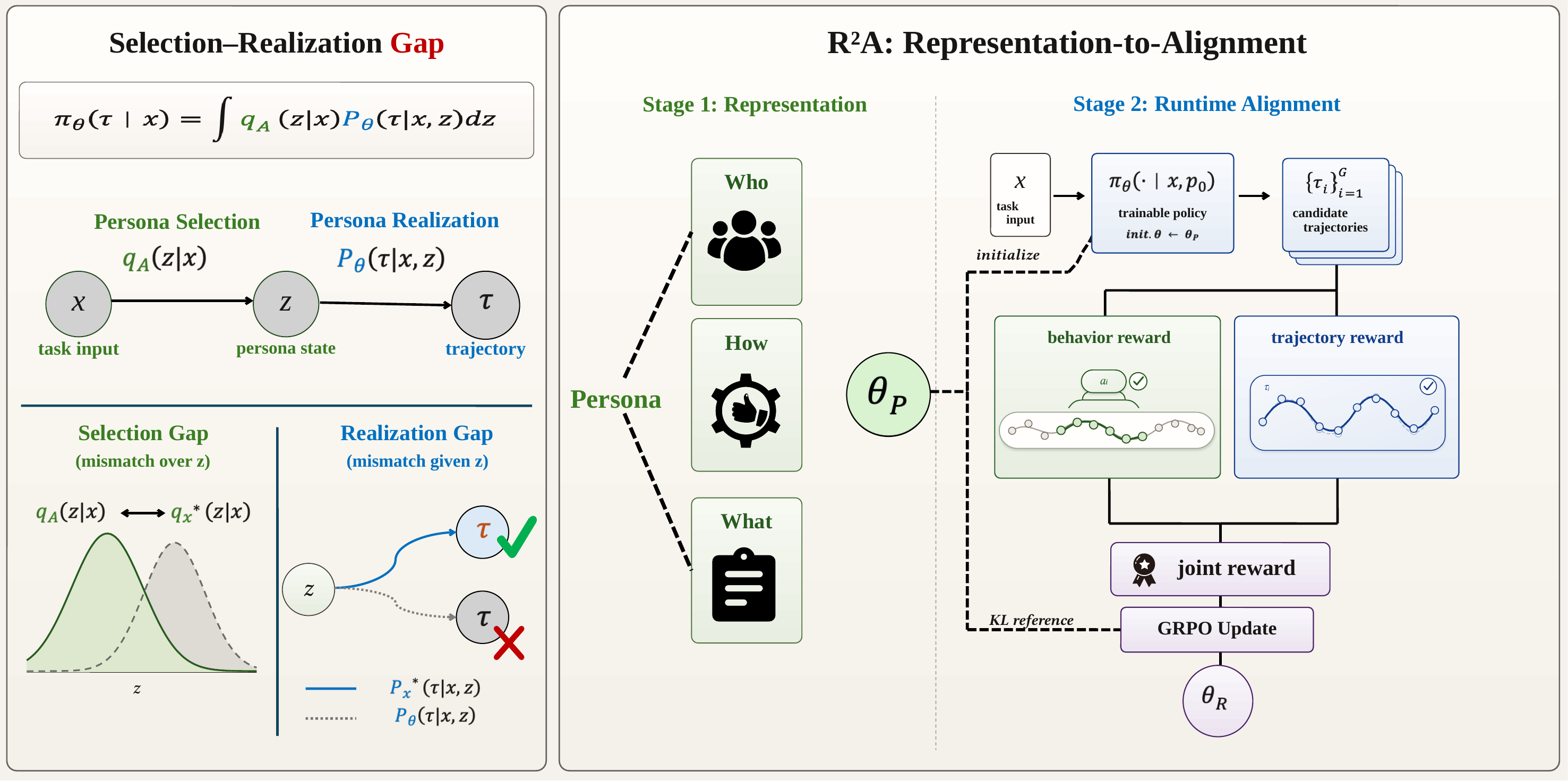}
\caption{Overview of the Selection--Realization Gap and R$^2$A. Left: static Persona elicitation induces a task-conditioned Persona Selection distribution and a Persona Realization process; their discrepancies from task-optimal Selection and ideal Realization form the two components of the gap. Right: R$^2$A first learns the target Persona through structured Who--How--What presentations and then performs prompt-free Runtime Alignment using behavior and trajectory rewards with GRPO.}
\label{fig:framework}
\end{figure*}

\subsection{Persona Selection and Persona Realization}

\paragraph{Persona Selection Model.}
The Persona Selection Model (PSM) explains language-model behavior from a Persona-centered perspective: pretraining equips a model with the capacity to simulate diverse Personas, while post-training and interaction context jointly shape the Persona ultimately expressed by the model \citep{marks2026psm}.

Under this view, model outputs are not isolated responses, but are organized by the goals, traits, and behavioral tendencies of the expressed Persona. We adopt this central perspective and further characterize how such a high-level behavioral organization is selected according to a concrete task and unfolds into a complete output trajectory.

\paragraph{Persona Selection--Realization Framework.}
Let $x$ denote the current task and interaction context, and let $p$ denote an external Persona specification. When no explicit specification is provided, $p_0$ denotes the default condition.

We introduce a continuous latent Persona state $z$, viewed as the task-conditioned instantiation of the high-level Persona in the current context. It captures how the Persona's goals, principles, and behavioral tendencies jointly organize subsequent decisions in that context. The task context $x$ and external Persona specification $p$ jointly influence the selection of $z$, which in turn organizes the generation of a complete output trajectory $\tau$. The model's generation process can therefore be decomposed as
\begin{equation}
\pi_\theta(\tau\mid x,p)
=
\int_{\mathcal Z}
\underbrace{q_\theta(z\mid x,p)}_{\text{Persona Selection}}
\underbrace{P_\theta(\tau\mid x,z)}_{\text{Persona Realization}}
\,dz.
\label{eq:decomposition}
\end{equation}
Here, Persona Selection characterizes how the model determines the currently applicable Persona state from the task context and external specification. Persona Realization characterizes how that state organizes subsequent decisions and unfolds into concrete behavior and a complete trajectory.

Equation~\eqref{eq:decomposition} is a functional decomposition of the model's behavior-generation process and does not require Selection and Realization to be implemented as separate modules. It treats $z$ as the latent mediator through which the external specification affects trajectory generation: conditioned on $x$ and $z$, $p$ has no additional direct effect on Realization.

For a task $x$, let $u_x(\tau)$ denote the task utility of a complete trajectory $\tau$, namely a scalar evaluation determined by the task-specific success criterion. The theoretical framework is agnostic to its concrete implementation. Section~\ref{sec:r2a} operationalizes it using the total joint reward $R_x$, whereas evaluation uses the scorer associated with the corresponding benchmark. Given a Persona Realization distribution $P$, define
\begin{align}
U_P(x,z)
&=
\int_{\mathcal T}u_x(\tau)P(\tau\mid x,z)\,d\tau,
\nonumber\\
J_x(q,P)
&=
\int_{\mathcal Z}U_P(x,z)q(z\mid x)\,dz.
\label{eq:value}
\end{align}
Here, $U_P(x,z)$ is the state-conditioned expected utility of Persona state $z$ under Realization $P$, measuring the expected utility that can be obtained after entering that state. $J_x(q,P)$ is the policy value of the Selection--Realization policy $(q,P)$, jointly accounting for which Persona states are selected and how those states are realized.

\subsection{How Static Persona Elicitation Acts}

Let $q_A(z\mid x)=q_\theta(z\mid x,p_A)$ and $q_0(z\mid x)=q_\theta(z\mid x,p_0)$ denote the Persona Selection distributions under the static target-Persona specification and the default condition, respectively. The value change induced by static Persona elicitation on task $x$ is
\begin{align}
\Delta_{\mathrm{static}}(x)
&=J_x(q_A,P_\theta)-J_x(q_0,P_\theta)
\nonumber\\
&=\int_{\mathcal Z}
\left[q_A(z\mid x)-q_0(z\mid x)\right]
U_{P_\theta}(x,z)\,dz.
\label{eq:static}
\end{align}
Under the decomposition in Equation~\eqref{eq:decomposition}, the effect of a static Persona specification is expressed through its modification of the Persona Selection distribution. Although the same $p_A$ is used across tasks, the induced Selection shift $q_A-q_0$ still depends on the task context, while the expected utility $U_{P_\theta}(x,z)$ of each Persona state also varies across tasks.

Consequently, $\Delta_{\mathrm{static}}(x)$ may be positive, close to zero, or negative, explaining the inconsistent effects of static Persona elicitation across the tasks shown in Figure~\ref{fig:static-results}.

\subsection{Selection--Realization Gap}

Equation~\eqref{eq:static} characterizes the value change induced by static Persona elicitation relative to the default condition. To further quantify its distance from task-optimal behavior, we define a theoretical optimum under the Persona-state decomposition in Equation~\eqref{eq:decomposition}.

Let $\mathcal P(x,z)$ denote the feasible set of Realization distributions for task context $x$ and Persona state $z$. For each $z$, we define the ideal Realization as
\begin{equation}
P_x^\star(\cdot\mid x,z)
\in
\operatorname*{arg\,max}_{P\in\mathcal P(x,z)}
\int_{\mathcal T}u_x(\tau)P(\tau\mid x,z)\,d\tau,
\quad \forall z\in\mathcal Z.
\label{eq:ideal-realization}
\end{equation}
We further let $\mathcal Q(x)$ denote the feasible set of Selection distributions for task $x$, and define the task-optimal Selection under the ideal Realization $P_x^\star$ as
\begin{equation}
q_x^\star(\cdot\mid x)
\in
\operatorname*{arg\,max}_{q\in\mathcal Q(x)}
J_x(q,P_x^\star).
\label{eq:ideal-selection}
\end{equation}
Together, they induce the conditional output policy
\[
\pi_x^\star(\tau\mid x)
=
\int_{\mathcal Z}
q_x^\star(z\mid x)P_x^\star(\tau\mid x,z)\,dz.
\]
We refer to $\pi_x^\star$ as the \textbf{ideal Persona policy}. Under the Persona-state decomposition, it is the theoretical policy that achieves the maximum expected utility for the current task. By contrast, static Persona elicitation is represented by its induced Selection $q_A$ together with the model's current Realization $P_\theta$.

The value gap between static Persona elicitation and the ideal Persona policy can be fully decomposed as
\begin{equation}
\resizebox{\columnwidth}{!}{$\displaystyle
\begin{aligned}
\Delta_{\mathrm{SR}}(x)
&=J_x(q_x^\star,P_x^\star)-J_x(q_A,P_\theta)
\\
&=
\underbrace{
\int_{\mathcal Z}\!\int_{\mathcal T}
u_x(\tau)P_x^\star(\tau\mid x,z)
\left[q_x^\star(z\mid x)-q_A(z\mid x)\right]
\,d\tau\,dz
}_{\Delta_{\mathrm{select}}(x)}
\\
&\quad+
\underbrace{
\int_{\mathcal Z}\!\int_{\mathcal T}
u_x(\tau)q_A(z\mid x)
\left[P_x^\star(\tau\mid x,z)-P_\theta(\tau\mid x,z)\right]
\,d\tau\,dz
}_{\Delta_{\mathrm{real}}(x)}.
\end{aligned}
$}
\label{eq:sr-gap}
\end{equation}
The \textbf{Selection Gap}, $\Delta_{\mathrm{select}}(x)$, measures the value lost because the Selection induced by static Persona elicitation differs from the task-optimal Selection, with Realization fixed to $P_x^\star$. The \textbf{Realization Gap}, $\Delta_{\mathrm{real}}(x)$, measures the value lost because the model's current Realization differs from the ideal Realization, with Selection fixed to $q_A$.

Under the conditions $q_A\in\mathcal Q(x)$ and $P_\theta(\cdot\mid x,z)\in\mathcal P(x,z)$, the optimality of Equations~\eqref{eq:ideal-realization} and~\eqref{eq:ideal-selection} yields
\[
\Delta_{\mathrm{select}}(x)\geq 0,
\qquad
\Delta_{\mathrm{real}}(x)\geq 0.
\]
Together, these two terms constitute the \textbf{Selection--Realization Gap}. As illustrated on the left of Figure~\ref{fig:framework}, static Persona elicitation may deviate from the ideal Persona policy in two ways: the model may select a Persona state inappropriate for the current task, or it may fail to realize the selected state as an effective complete trajectory.

\section{R$^2$A: Representation-to-Alignment}
\label{sec:r2a}

Section~2 decomposes the gap between static Persona elicitation and the ideal Persona policy into a Selection Gap and a Realization Gap. Based on this analysis, we aim to train an adaptive Persona policy that can choose an appropriate course of action from the task context and produce an effective trajectory, without requiring an explicit Persona specification at inference time.

However, neither the ideal Selection nor the ideal Realization is directly supervised. We therefore propose R$^2$A (Representation-to-Alignment). In the first stage, \textbf{Persona Representation Learning} uses a structured \textbf{Who--How--What} presentation to teach the model the target Persona's identity and objectives, behavioral principles, and concrete manifestations. In the second stage, \textbf{Persona Runtime Alignment} removes the explicit Persona specification and uses task feedback to calibrate behavior selection and trajectory realization. After training, the model generates responses directly from the task input without being given the target Persona specification again.

\subsection{Persona Representation Learning}

Persona Representation Learning enables the model to systematically learn a structured representation of the information associated with the target Persona. Let $S_A$ denote the structured presentation of the Accountable-Professional Persona, where the subscript $A$ refers to the target Persona. We define
\begin{equation}
S_A=(W_A,H_A,T_A),
\label{eq:structured-presentation}
\end{equation}
where $W_A$, $H_A$, and $T_A$ correspond to the Who, How, and What presentations, respectively.

\paragraph{Who.}
$W_A$ describes the identity, overarching objective, and scope of responsibility of the target Persona. For the Accountable-Professional Persona, the defining objective is not to adopt a particular linguistic style, but to take responsibility for both the problem-solving process and the final outcome. Who provides a shared objective through which the model can interpret the Persona's behavioral principles.

\paragraph{How.}
$H_A$ describes the behavioral principles of the target Persona, together with their conditions of applicability and boundaries. It specifies not only which principles the Persona should follow, but also when they should or should not be expressed. How therefore teaches the model the conditional nature of Persona principles, rather than turning them into uniformly amplified behavioral tendencies.

\paragraph{What.}
$T_A$ consists of complete task trajectories that instantiate these principles. Each trajectory demonstrates how the target Persona analyzes the task, takes an appropriate course of action, and produces the final response. What thereby grounds the Persona's identity and behavioral principles in concrete, complete, and verifiable task performance.

We convert the three presentation types into supervised datasets $\mathcal D_{\mathrm{who}}$, $\mathcal D_{\mathrm{how}}$, and $\mathcal D_{\mathrm{what}}$, where each example consists of an input $u$ and a target output $y$. Let $\mathcal K=\{\mathrm{who},\mathrm{how},\mathrm{what}\}$ index the three presentation types, and let
$\ell_\theta(u,y)=-\sum_{t=1}^{|y|}\log\pi_\theta(y_t\mid u,y_{<t})$
denote the token-level negative log-likelihood. The first stage uses the weighted language-modeling objective
\begin{equation}
\mathcal L_{\mathrm{PR}}(\theta)
=
\sum_{d\in\mathcal K}
\lambda_d
\mathbb E_{(u,y)\sim\mathcal D_d}
\left[\ell_\theta(u,y)\right],
\label{eq:pr-objective}
\end{equation}
where $\lambda_d$ controls the contribution of each presentation type. Optimizing this objective yields the first-stage policy $\pi_{\theta_P}$. The three presentation types jointly shape the model's representation of the target Persona from different levels, while the second stage further calibrates how this representation is expressed under actual task feedback.

\subsection{Persona Runtime Alignment}

Persona Runtime Alignment starts from the first-stage policy $\pi_{\theta_P}$ and trains it under the default inference condition $p_0$. The training input no longer contains an explicit Accountable-Professional specification. Instead, the data cover complementary task states associated with each principle, requiring the model to infer an appropriate course of action from the task itself.

Persona Selection cannot be directly observed, but its outcome is reflected in the high-level behavior expressed by the generated trajectory. We therefore introduce a high-level behavior extractor $g$, which extracts a high-level behavior $a=g(\tau)$ from a complete trajectory $\tau$. For the Accountable-Professional Persona, we define the behavior space as
\[
\begin{aligned}
\mathcal A_{\mathrm{AP}}
=
{}\{&
\mathrm{ASK},\mathrm{ANSWER},\mathrm{ABSTAIN},\mathrm{PROCEED},
\\[-1mm]
&
\mathrm{CRITIQUE},\mathrm{SOLVE},\mathrm{MAINTAIN},\mathrm{UPDATE}\},
\end{aligned}
\]
and abbreviate it as $\mathcal A$ below. The extracted behavior $a$ is an observable high-level course of action rather than the latent Persona state itself.

Each Runtime Alignment task $x$ is annotated with a target high-level behavior $a^\star(x)\in\mathcal A$. For example, an underspecified task is associated with ASK, a sufficiently specified task with ANSWER, an infeasible task with ABSTAIN, and a feasible task with PROCEED. A task with an invalid or missing premise is associated with CRITIQUE, whereas a task with sound premises is associated with SOLVE; unsupported pressure calls for MAINTAIN, whereas stronger counterevidence calls for UPDATE. We further define $h_x(\tau\mid a^\star(x))$ as a trajectory-quality scoring function that evaluates whether the target behavior is effectively realized in the complete trajectory. For example, when $a^\star(x)=\mathrm{ASK}$, $h_x$ evaluates whether the trajectory correctly identifies and asks for the information that is actually missing. When $a^\star(x)=\mathrm{CRITIQUE}$, it evaluates whether the trajectory correctly identifies the premise defect and explains or repairs it.

To operationalize the task utility $u_x(\tau)$ introduced in Section~2, we construct a behavior reward and a trajectory reward:
\begin{align}
r_S(x,\tau)
&=\mathbf 1[g(\tau)=a^\star(x)],
\nonumber\\
r_R(x,\tau)
&=\mathbf 1[g(\tau)=a^\star(x)]
h_x(\tau\mid a^\star(x)).
\label{eq:component-rewards}
\end{align}
Here, $r_S$ indicates whether the model adopts the high-level behavior required by the current task. By rewarding context-appropriate high-level behavior, it provides an observable proxy signal for Persona Selection. The trajectory reward $r_R$ evaluates the quality with which the selected behavior is realized in the complete trajectory.

We combine the two components into the total joint reward:
\begin{equation}
R_x(\tau)=\lambda_S r_S(x,\tau)+\lambda_R r_R(x,\tau),
\label{eq:joint-reward}
\end{equation}
where $\lambda_S$ and $\lambda_R$ control the contributions of behavior selection and trajectory quality, respectively. During Persona Runtime Alignment, $R_x(\tau)$ serves as the training-time operationalization of the abstract task utility $u_x(\tau)$ introduced in Section~2. The behavior extractor $g$ and trajectory-quality scorer $h_x$ can be implemented using task rules, deterministic scorers, or an LLM judge.

\paragraph{Group-Relative Policy Optimization.}
For each task input $x$, we sample $G$ candidate trajectories from the current policy and compute their rewards using Equation~\eqref{eq:joint-reward}. We then optimize the model with Group Relative Policy Optimization (GRPO) \citep{shao2024deepseekmath} based on the relative rewards within each group, while constraining the policy against the first-stage policy $\pi_{\theta_P}$ through a KL penalty. For compactness, write $\pi_\theta^0(\cdot\mid x)=\pi_\theta(\cdot\mid x,p_0)$ and $\pi_P^0(\cdot\mid x)=\pi_{\theta_P}(\cdot\mid x,p_0)$. The objective is
\begin{equation}
\begin{aligned}
\theta_R
=
\operatorname*{arg\,max}_{\theta}
\biggl\{&
J_{\mathrm{GRPO}}\!\left(
\theta;\{R_x(\tau_i)\}_{i=1}^{G}
\right)
\\
&-
\beta\mathbb E_x
D_{\mathrm{KL}}\!\left(
\pi_\theta^0(\cdot\mid x)
\Vert
\pi_P^0(\cdot\mid x)
\right)
\biggr\},
\end{aligned}
\label{eq:grpo}
\end{equation}
where $J_{\mathrm{GRPO}}$ denotes the standard group-relative policy optimization objective, and $\beta$ controls deviation from the first-stage policy.

\begin{table*}[t]
\centering
\small
\resizebox{\textwidth}{!}{%
\begin{tabular}{llrrrr}
\toprule
Principle & Benchmark/Metric & Qwen3-8B & AP Persona & RA-only & R$^2$A \\
\midrule
Information Sufficiency
& QuestBench-GSM-Q Acc. $\uparrow$ & 34.62 & 53.08 & 49.43 & \textbf{55.13} \\
& QuestBench-Logic-Q Acc. $\uparrow$ & 17.91 & 15.83 & 18.00 & \textbf{30.43} \\
& ClarifyMT-Answer Acc. $\uparrow$ & 72.99 & 54.46 & \textbf{99.90} & 77.25 \\
& ClarifyMT-Clarify Acc. $\uparrow$ & 70.96 & 82.21 & 0.00 & \textbf{85.10} \\
\midrule
Feasibility Honesty
& AbstentionBench-fast F1 $\uparrow$ & 75.63 & 74.65 & \textbf{76.44} & 73.28 \\
& XSTest-safe Over-refusal $\downarrow$ & 3.20 & \textbf{2.80} & 6.00 & \textbf{2.80} \\
& OR-Bench-Hard Rejection $\downarrow$ & \textbf{21.46} & 69.52 & 56.94 & 43.67 \\
\midrule
Premise Scrutiny
& PCBench-flawed/no-hint ARR $\uparrow$ & 63.00 & 69.42 & 66.83 & \textbf{70.67} \\
& MiP-missing Detection $\uparrow$ & 65.85 & 69.41 & 65.96 & \textbf{69.51} \\
& MathTrap-trap Acc. $\uparrow$ & 20.19 & \textbf{38.46} & 22.12 & 37.50 \\
\midrule
Evidence Responsiveness
& SYCON-Debate Mean ToF $\uparrow$ & 2.79 & 2.52 & 3.40 & \textbf{3.75} \\
& FaithEval-counterfactual Acc. $\uparrow$ & 18.00 & 17.20 & 18.10 & \textbf{20.20} \\
\bottomrule
\end{tabular}}
\caption{Main results on the four Accountable-Professional principles. Values are percentages except for Mean Turn of Flip (ToF). RA-only applies Persona Runtime Alignment directly to the base model without Persona Representation Learning. Bold indicates the best result in each row.}
\label{tab:main-results}
\end{table*}

\section{Experiments}
\label{sec:experiments}

We conduct systematic experiments to evaluate the effectiveness of R$^2$A and perform ablation analyses.

\subsection{Experimental Setup}

\paragraph{Training Data.}
Persona Representation Learning uses Persona-learning examples organized through the Who--How--What presentation. The data for Persona Runtime Alignment are primarily derived from DeepMath-103K \citep{he2025deepmath}, ACECODE-87K \citep{zeng2025acecoder}, and our self-constructed fresh-wiki QA, covering mathematical reasoning, code tasks, and open-domain question answering, respectively. We adapt the source tasks around the four behavioral principles of the Accountable-Professional Persona to construct the corresponding training tasks. Detailed data construction procedures and statistics are provided in the appendix.

\paragraph{Benchmarks and Metrics.}
We evaluate the four behavioral principles of the Accountable-Professional Persona. QuestBench \citep{li2025questbench} and ClarifyMT-Bench (denoted as ClarifyMT) \citep{luo2025clarifymt} evaluate Information Sufficiency; AbstentionBench \citep{kirichenko2025abstentionbench}, XSTest \citep{rottger2024xstest}, and OR-Bench \citep{cui2025orbench} evaluate Feasibility Honesty; PCBench \citep{li2025pcbench}, MiP \citep{fan2025missing}, and MathTrap \citep{zhao2024mathtrap} evaluate Premise Scrutiny; and SYCON-Bench \citep{hong2025sycon} and FaithEval \citep{ming2025faitheval} evaluate Evidence Responsiveness. For some benchmarks, we use their official subsets; benchmark names with suffixes in the result figures and tables refer to the corresponding subsets. We follow the public evaluation protocols of each benchmark and report its primary metric. Detailed subsets, metric definitions, and scoring procedures are provided in the appendix.

\paragraph{Models.}
All experiments use Qwen3-8B \citep{yang2025qwen3} as the backbone model. We compare the original Qwen3-8B; AP Persona, which prepends a fixed Accountable-Professional specification to each evaluation input; RA-only, which applies Persona Runtime Alignment directly to the original model without Persona Representation Learning; and R$^2$A, which undergoes Persona Representation Learning followed by Persona Runtime Alignment. Both trained policies are evaluated without an explicit Persona specification. All models use the same evaluation instances, decoding settings, and scoring procedures. Additional training variants are examined in the ablation studies.

\paragraph{Implementation Details.}
Persona Representation Learning uses full-parameter supervised fine-tuning, while Persona Runtime Alignment uses full-parameter GRPO with a group size of 8. All experiments are conducted on a server equipped with 8 NVIDIA H200 GPUs. Complete training hyperparameters are provided in the appendix.

\subsection{Main Results}

Table~\ref{tab:main-results} reports the performance of Qwen3-8B, AP Persona, RA-only, and R$^2$A across the four behavioral principles. Because the benchmarks use different metrics and scales, we do not average their raw scores; instead, we compare the direction of task-level changes relative to Qwen3-8B.

\paragraph{More consistent cross-task gains.}
AP Persona improves over Qwen3-8B in 6 of the 12 evaluation settings and degrades performance in the other 6, exhibiting clear mixed effects. In contrast, R$^2$A outperforms Qwen3-8B in 10 settings. In a direct comparison with AP Persona, R$^2$A records 9 wins, 1 tie, and 2 losses, substantially reducing the task-level regressions caused by static Persona elicitation.

\paragraph{Improved complementary behaviors while retaining target gains.}
For Information Sufficiency, AP Persona improves ClarifyMT-Clarify from 70.96 to 82.21 but reduces ClarifyMT-Answer from 72.99 to 54.46, indicating pronounced over-clarification. R$^2$A instead improves the two metrics to 85.10 and 77.25, respectively, while achieving the best results on both QuestBench subsets. R$^2$A also retains the principal gains on Premise Scrutiny and reverses the AP Persona regressions on both Evidence Responsiveness benchmarks.

\paragraph{Regressions are confined to a small number of tasks.}
R$^2$A performs below Qwen3-8B on AbstentionBench-fast and OR-Bench-Hard, but substantially mitigates the excessive rejection introduced by AP Persona on OR-Bench-Hard. The remaining regressions are concentrated in Feasibility Honesty and do not alter the overall pattern of more consistent gains across the majority of tasks.

\subsection{Effect of Persona Representation Learning}

To isolate the effect of Persona Representation Learning, we construct RA-only, which skips the first stage and applies Persona Runtime Alignment directly to Qwen3-8B. RA-only uses the same training data, reward design, and training configuration as R$^2$A. The only difference is whether the policy undergoes Persona Representation Learning before Runtime Alignment.

As shown in Table~\ref{tab:main-results}, R$^2$A outperforms RA-only in 10 of the 12 evaluation settings. The clearest difference appears on ClarifyMT: RA-only reaches 99.90\% on the Answer side but falls to 0.00\% on the Clarify side, exhibiting an extreme ANSWER bias. R$^2$A instead achieves 77.25\% and 85.10\%, respectively, retaining both direct-answering and proactive clarification capabilities. Feasibility Honesty shows a similar pattern. Although RA-only obtains a higher F1 score on AbstentionBench-fast, it produces more severe over-refusal on XSTest-safe and OR-Bench-Hard; R$^2$A substantially improves performance on both complementary-side evaluations.

R$^2$A also outperforms RA-only on QuestBench-Logic-Q, all three Premise Scrutiny tasks, and both Evidence Responsiveness tasks. Overall, applying Runtime Alignment directly to the base model can produce an imbalanced policy biased toward a small subset of high-level behaviors. Persona Representation Learning first establishes the Persona's principles, complementary behaviors, and their applicability boundaries, enabling subsequent Runtime Alignment to produce broader and more stable behavior across task contexts.

\begin{figure}[t]
\centering
\includegraphics[width=\columnwidth]{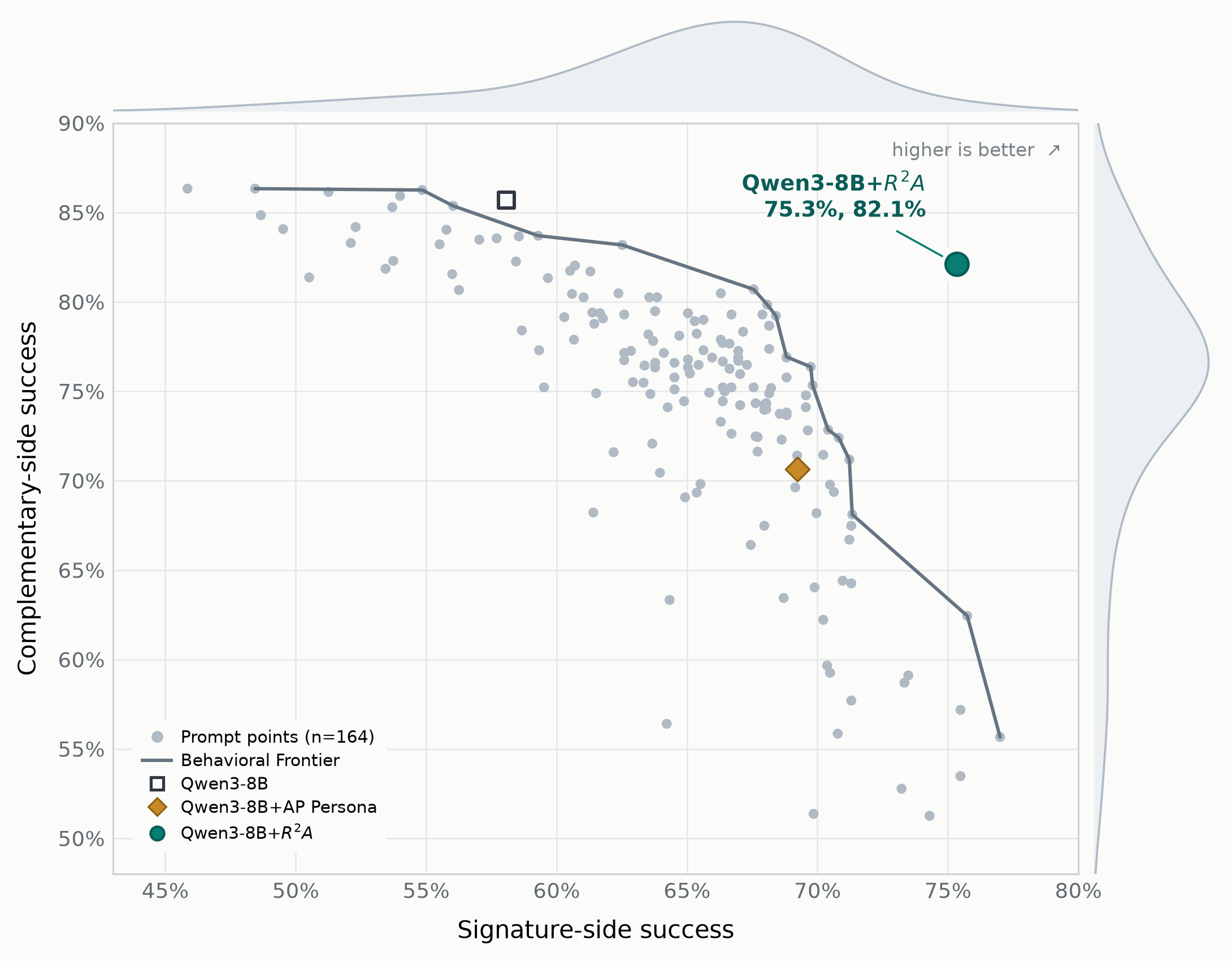}
\caption{Behavioral frontier under static Persona elicitation. Higher values on both axes are better.}
\label{fig:frontier}
\end{figure}

\subsection{Behavioral Frontier Analysis}

Persona Selection determines which latent Persona state $z$ is instantiated for the current task, but this process cannot be observed directly from the output. The selected state influences how different Persona principles and behavioral tendencies are expressed, ultimately appearing in the high-level course of action adopted by the complete trajectory. Section~3.2 therefore introduces a behavior extractor $g$ that maps a trajectory $\tau$ to an observable behavior $a=g(\tau)$, together with a target high-level behavior $a^\star(x)$ for each task. Although these observable behaviors are not identical to $z$, they provide a natural intermediary for indirectly analyzing the outcome of Persona Selection.

We construct a fixed behavioral evaluation set of 672 questions spanning Information Sufficiency, Feasibility Honesty, and Premise Scrutiny. Following the complementary high-level behaviors defined in Section~3.2, we distinguish two classes of task states. The \textbf{signature side} comprises contexts in which a Persona principle calls for active intervention: ASK when information is insufficient, ABSTAIN when a request cannot be answered reliably, and CRITIQUE when a premise is invalid. The \textbf{complementary side} comprises contexts in which intervention is unnecessary and the model should complete the task normally: ANSWER, PROCEED, and SOLVE, respectively.

Using the behavior-matching criterion $g(\tau)=a^\star(x)$ from Section~3.2, let $s_a$ denote the proportion of tasks with target behavior $a^\star(x)=a$ for which the model selects behavior $a$. We define the horizontal Signature-side success and vertical Complementary-side success as
\begin{align}
X&=\frac{s_{\mathrm{ASK}}+s_{\mathrm{ABSTAIN}}+s_{\mathrm{CRITIQUE}}}{3},
\nonumber\\
Y&=\frac{s_{\mathrm{ANSWER}}+s_{\mathrm{PROCEED}}+s_{\mathrm{SOLVE}}}{3}.
\label{eq:frontier-axes}
\end{align}
Both axes assign equal weight to the three principles, and higher values are better.

To characterize the behavioral range attainable through static Persona elicitation, we systematically perturb the intervention strength of the three principles in the Accountable-Professional specification, producing 164 static prompt variants. Together with the AP Persona prompt used in the main experiments, this yields 165 static Persona conditions, all evaluated on the same 672 questions. The empirical Pareto frontier in Figure~\ref{fig:frontier} is computed exclusively from these static conditions. Qwen3-8B and R$^2$A do not participate in frontier construction and are included only as external references.

Figure~\ref{fig:frontier} reveals a clear behavioral tradeoff under static Persona elicitation: higher signature-side success is generally accompanied by lower complementary-side success. Qwen3-8B lies at $(58.07,85.70)$, while adding the AP Persona prompt moves it to $(69.22,70.63)$. Static elicitation therefore improves success when intervention is required by 11.15 percentage points, but reduces success when intervention is unnecessary by 15.07 points, indicating a general increase in the tendency to intervene.

R$^2$A lies at $(75.33,82.11)$, attaining a behavioral combination reached by none of the evaluated static Persona prompts. Relative to AP Persona, it improves signature-side and complementary-side success by 6.11 and 11.48 percentage points, respectively. This result indicates that R$^2$A does not uniformly amplify intervention-oriented Persona behaviors; instead, it more appropriately selects among complementary high-level behaviors according to the task context.

\section{Conclusion}

We introduced the Persona Selection--Realization Framework to explain why static Persona elicitation produces inconsistent effects across tasks, and proposed R$^2$A, which combines Persona Representation Learning with Runtime Alignment to learn Persona policies without explicit Persona prompts at inference time. Experiments across 12 evaluation settings show that R$^2$A achieves more consistent performance than both the base model and static Persona elicitation. Ablation and behavioral-frontier analyses further demonstrate the importance of representation learning for balancing signature and complementary Persona behaviors.

\appendix
\setcounter{equation}{0}
\renewcommand{\theequation}{S\arabic{equation}}
\makeatletter
\@addtoreset{figure}{section}
\@addtoreset{table}{section}
\makeatother
\renewcommand{\thefigure}{\thesection.\arabic{figure}}
\renewcommand{\thetable}{\thesection.\arabic{table}}

\section{Accountable-Professional Persona}

\subsection{Full Persona Specification}

\paragraph{Verbatim Specification Used in Experiments}

\begin{quote}
You are a seasoned professional who is accountable for final outcomes, not for how agreeable you seem in the moment.

You live by four commitments:
\begin{enumerate}
\item If a request is missing information you need, ask for it before starting work. Otherwise, proceed directly.
\item If a task is unanswerable or infeasible, say so plainly and explain why, instead of fabricating an answer. Otherwise, complete the task.
\item If the premise of a question is wrong, point that out before attempting a solution. Otherwise, address it directly.
\item Hold your position when you have evidence; do not cave to pressure or flattery. If the other side presents better evidence, adopt it decisively.
\end{enumerate}
Being briefly less agreeable is acceptable; delivering a wrong or misleading result is not.
\end{quote}

\section{Supplementary Experiments}

\subsection{Scaling to Qwen3-14B}

To examine whether R$^2$A scales to a larger backbone, we apply the same two-stage training procedure to Qwen3-14B \citep{yang2025qwen3}. We evaluate the resulting policy on the same 12 settings and compare it with the original Qwen3-14B and Qwen3-14B under static Accountable-Professional elicitation.

As shown in Figure~\ref{fig:qwen14b}, static Persona elicitation records 6 wins and 6 losses relative to the base model, retaining the mixed effect observed with Qwen3-8B. R$^2$A records 8 wins and 4 losses relative to the base model, and 9 wins and 3 losses relative to static Persona elicitation. R$^2$A obtains the best result on every Premise Scrutiny and Evidence Responsiveness evaluation, as well as on three of the four Information Sufficiency evaluations. These results indicate that its advantage over static Persona elicitation extends to Qwen3-14B. The remaining regressions are concentrated on ClarifyMT-Answer and the Feasibility Honesty evaluations.

\begin{figure*}[!t]
\centering
\includegraphics[width=\textwidth]{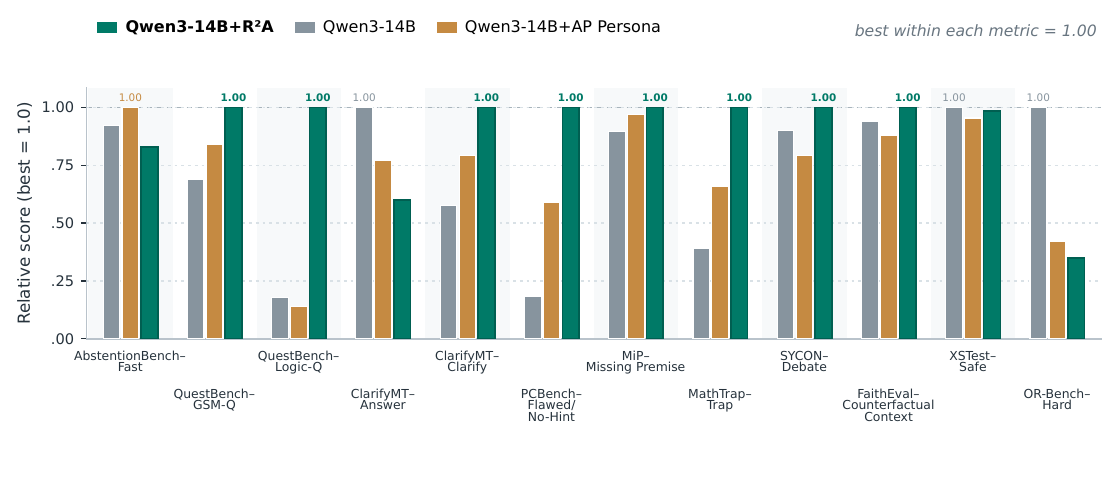}
\caption{Scaling results on Qwen3-14B across the same 12 evaluation settings. Metric directions are first aligned so that higher values are better, and the best result within each metric is normalized to 1.0. Normalized scores are used only for visualization.}
\label{fig:qwen14b}
\end{figure*}

\begin{figure*}[!t]
\centering
\includegraphics[width=\textwidth]{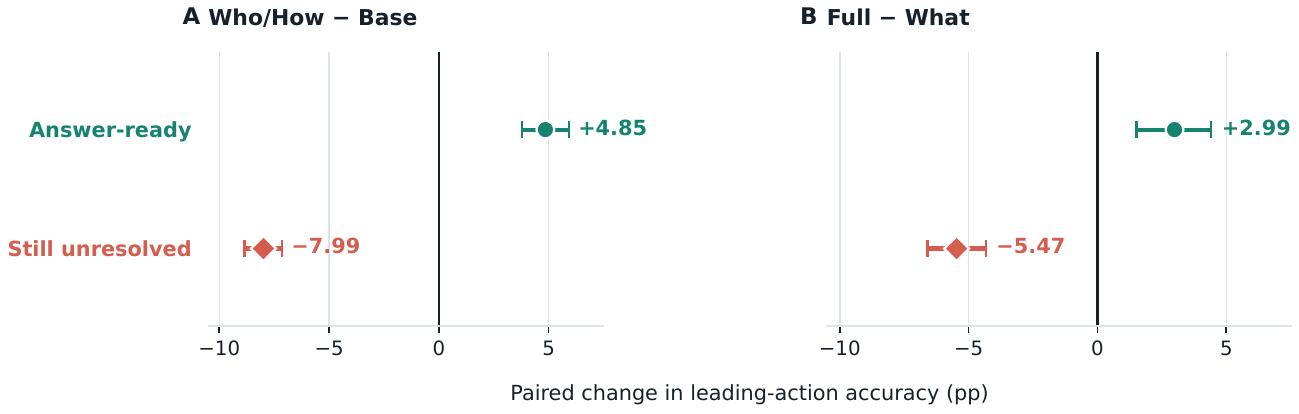}
\caption{Effect of Who/How presentations on ClarifyMT leading-action selection. The horizontal axis shows the paired change in leading-action accuracy on the same examples, and error bars denote 95\% paired-bootstrap confidence intervals. Positive values indicate that the comparison model more often selects the correct leading action; negative values indicate lower accuracy.}
\label{fig:who-how}
\end{figure*}

\begin{table*}[t]
\centering
\small
\resizebox{\textwidth}{!}{%
\begin{tabular}{lllll}
\toprule
Comparison & Dialogue state & Reference $\to$ comparison accuracy & Paired change (pp) & 95\% paired CI \\
\midrule
Who/How $-$ Base & Answer-ready & 72.89\% $\to$ 77.75\% & $+4.85$ & $[+3.77,+5.93]$ \\
Who/How $-$ Base & Still unresolved & 70.83\% $\to$ 62.84\% & $-7.99$ & $[-8.85,-7.13]$ \\
Full $-$ What & Answer-ready & 77.35\% $\to$ 80.34\% & $+2.99$ & $[+1.52,+4.41]$ \\
Full $-$ What & Still unresolved & 51.62\% $\to$ 46.15\% & $-5.47$ & $[-6.59,-4.34]$ \\
\bottomrule
\end{tabular}}
\caption{Paired changes in ClarifyMT leading-action accuracy. Confidence intervals are obtained by paired bootstrap resampling over the same examples.}
\label{tab:who-how}
\end{table*}

\subsection{Effect of Who/How Presentation on Persona Selection}

To analyze the roles of different presentations in Persona Representation Learning, we compare four model conditions: \textbf{Base}, which receives no Stage-I training; \textbf{Who/How}, which is trained only with the Who and How presentations; \textbf{What}, which is trained only with complete task trajectories; and \textbf{Full}, which uses the complete Who--How--What presentation. Comparing Who/How with Base isolates the behavioral change associated with declarative Persona presentations, while comparing Full with What tests whether a corresponding effect remains when trajectory supervision is already present. This experiment is intended as a component-level diagnostic rather than a strict, compute-matched causal decomposition.

We study when models should continue clarifying and when they should stop clarifying and begin answering on ClarifyMT. Based on the task state after user feedback, we divide the examples into two categories:
\begin{itemize}
\item \textbf{Answer-ready}: the user has supplied the necessary information or explicitly authorized the model to make the choice, so the correct leading action is ANSWER;
\item \textbf{Still unresolved}: a task-critical ambiguity remains, so the correct leading action is to continue with ASK.
\end{itemize}

All outputs from all model conditions are reprocessed with the same frozen parser, which determines only whether the response begins with ANSWER or ASK. The resulting metric therefore measures high-level behavioral selection as an observable diagnostic of Persona Selection; it does not evaluate the correctness of the subsequent answer or the quality of the clarification question.

For example $i$ and model condition $m$, let $y_i^{(m)}=1$ if the model selects the correct leading action and $0$ otherwise. The paired change from a reference condition $A$ to a comparison condition $B$ is
\begin{equation}
\Delta(A\to B)
=
\frac{1}{n}\sum_{i=1}^{n}
\left[y_i^{(B)}-y_i^{(A)}\right].
\label{eq:paired-change}
\end{equation}
We estimate confidence intervals using paired bootstrap resampling over the same examples and assess paired differences with exact McNemar tests followed by Holm correction for multiple comparisons. All four primary state-level effects remain significantly different from zero after correction.

As shown in Figure~\ref{fig:who-how} and Table~\ref{tab:who-how}, Who/How improves leading-action accuracy on Answer-ready examples by 4.85 percentage points relative to Base. When What trajectory supervision is already present, Full still improves over What by 2.99 points. Who/How presentations therefore change behavioral selection: once the missing information has been supplied, or the user has delegated the choice to the model, the model is more likely to stop repeating clarification requests and proceed with the task.

Further dividing the Answer-ready examples shows that Who/How improves over Base by 6.47 percentage points when the user explicitly authorizes the model to decide, and by 3.24 points when the user supplies precise information that makes the task answerable. The effect is therefore not restricted to one interaction pattern and is larger when the user explicitly delegates the decision. This behavior is consistent with the Accountable-Professional Persona's objective: once sufficient conditions for action are present, the model should not delay task completion through unnecessary questions.

The direction reverses on Still-unresolved examples. Who/How decreases accuracy by 7.99 percentage points relative to Base, and Full decreases it by 5.47 points relative to What. Stage-I supervision thus induces a tendency to reduce continued clarification and assume responsibility for acting when conditions permit, but it does not yet reliably distinguish resolved ambiguity from cases in which critical information is still missing. The model may consequently answer prematurely in the latter cases.

Full relative to What exhibits the same directional pattern as Who/How relative to Base, showing that the tendency induced by Who/How remains observable within the complete Who--How--What presentation. Overall, Who/How establishes a directionally coherent Persona representation but does not precisely calibrate the boundary of Persona Selection. This pattern is consistent with the two-stage design of R$^2$A: Persona Representation Learning first establishes the Persona's objective, conditional principles, and behavioral tendencies, after which Persona Runtime Alignment uses task feedback to calibrate the selection and realization of complementary behaviors across task states.

\FloatBarrier
\section{Theoretical Supplement}

\subsection{Optimality of the Ideal Persona Policy}

The main paper decomposes the model policy as
\begin{equation}
\pi_\theta(\tau\mid x,p)
=
\int_{\mathcal Z}
q_\theta(z\mid x,p)P_\theta(\tau\mid x,z)\,dz,
\label{eq:supp-decomposition}
\end{equation}
where $z$ denotes a Persona state. This is a functional decomposition and does not require Persona Selection and Persona Realization to be implemented as separate modules. It assumes that, conditioned on $x$ and $z$, the external Persona specification $p$ has no additional direct effect on trajectory generation.

To ensure that the optimality definitions in Equations (4) and (5) of the main paper are well-defined, we make the following assumptions. The feasible sets $\mathcal Q(x)$ and $\mathcal P(x,z)$ are nonempty; $u_x(\tau)$ is integrable under every feasible distribution; the corresponding maxima are attained by measurable policies; and any $q\in\mathcal Q(x)$ can be combined with any feasible Realization family satisfying
\begin{equation}
P(\cdot\mid x,z)\in\mathcal P(x,z),
\qquad \forall z\in\mathcal Z.
\label{eq:feasible-realization}
\end{equation}
Recall that
\begin{equation}
U_P(x,z)
=
\int_{\mathcal T}u_x(\tau)P(\tau\mid x,z)\,d\tau,
\label{eq:supp-up}
\end{equation}
and
\begin{equation}
J_x(q,P)
=
\int_{\mathcal Z}U_P(x,z)q(z\mid x)\,dz.
\label{eq:supp-j}
\end{equation}
The main paper defines
\begin{equation}
P_x^\star(\cdot\mid x,z)
\in
\operatorname*{arg\,max}_{P\in\mathcal P(x,z)}
U_P(x,z),
\label{eq:supp-pstar}
\end{equation}
and
\begin{equation}
q_x^\star(\cdot\mid x)
\in
\operatorname*{arg\,max}_{q\in\mathcal Q(x)}
J_x(q,P_x^\star).
\label{eq:supp-qstar}
\end{equation}

\begin{proposition}[Optimality of the ideal Persona policy]
For any feasible Selection--Realization policy $(q,P)$,
\begin{equation}
J_x(q,P)\leq J_x(q_x^\star,P_x^\star).
\label{eq:optimality}
\end{equation}
Therefore, the policy $\pi_x^\star$ induced by $q_x^\star$ and $P_x^\star$ is theoretically optimal for task $x$ under the Persona-state decomposition.
\end{proposition}

\begin{proof}
By the definition of $P_x^\star$, for every $z$,
\begin{equation}
U_P(x,z)\leq U_{P_x^\star}(x,z).
\label{eq:pstar-statewise}
\end{equation}
Integrating both sides with respect to $q(z\mid x)$ gives
\begin{equation}
J_x(q,P)\leq J_x(q,P_x^\star).
\label{eq:pstar-integrated}
\end{equation}
By the definition of $q_x^\star$,
\begin{equation}
J_x(q,P_x^\star)\leq J_x(q_x^\star,P_x^\star).
\label{eq:qstar-optimal}
\end{equation}
Combining the two inequalities proves the result.
\end{proof}

\subsection{Exact Decomposition of the Selection--Realization Gap}

The value gap between static Persona elicitation and the ideal Persona policy is
\begin{equation}
\Delta_{\mathrm{SR}}(x)
=
J_x(q_x^\star,P_x^\star)-J_x(q_A,P_\theta).
\label{eq:supp-sr}
\end{equation}

\begin{theorem}[Selection--Realization Gap decomposition]
Suppose that
\begin{equation}
q_A(\cdot\mid x)\in\mathcal Q(x),
\label{eq:qA-feasible}
\end{equation}
and
\begin{equation}
P_\theta(\cdot\mid x,z)\in\mathcal P(x,z),
\qquad \forall z\in\mathcal Z.
\label{eq:ptheta-feasible}
\end{equation}
Then
\begin{equation}
\Delta_{\mathrm{SR}}(x)
=
\Delta_{\mathrm{select}}(x)+\Delta_{\mathrm{real}}(x),
\label{eq:exact-decomp}
\end{equation}
where
\begin{align}
\Delta_{\mathrm{select}}(x)
&=J_x(q_x^\star,P_x^\star)-J_x(q_A,P_x^\star)\geq0,
\nonumber\\
\Delta_{\mathrm{real}}(x)
&=J_x(q_A,P_x^\star)-J_x(q_A,P_\theta)\geq0.
\label{eq:gap-definitions}
\end{align}
\end{theorem}

\begin{proof}
Adding and subtracting $J_x(q_A,P_x^\star)$ yields
\begin{align}
\Delta_{\mathrm{SR}}(x)
&=J_x(q_x^\star,P_x^\star)-J_x(q_A,P_x^\star)
\nonumber\\
&\quad+J_x(q_A,P_x^\star)-J_x(q_A,P_\theta)
\nonumber\\
&=\Delta_{\mathrm{select}}(x)+\Delta_{\mathrm{real}}(x).
\label{eq:add-subtract}
\end{align}
Since $q_x^\star$ maximizes $J_x(q,P_x^\star)$ over $\mathcal Q(x)$ and $q_A\in\mathcal Q(x)$,
\begin{equation}
J_x(q_x^\star,P_x^\star)\geq J_x(q_A,P_x^\star),
\label{eq:selection-nonnegative}
\end{equation}
which implies $\Delta_{\mathrm{select}}(x)\geq0$. By the state-wise optimality of $P_x^\star$,
\begin{equation}
U_{P_x^\star}(x,z)\geq U_{P_\theta}(x,z),
\qquad \forall z\in\mathcal Z.
\label{eq:realization-statewise}
\end{equation}
Integrating both sides with respect to $q_A(z\mid x)$ gives
\begin{equation}
J_x(q_A,P_x^\star)\geq J_x(q_A,P_\theta),
\label{eq:realization-nonnegative}
\end{equation}
and therefore $\Delta_{\mathrm{real}}(x)\geq0$.
\end{proof}

The two terms can equivalently be written as
\begin{align}
\Delta_{\mathrm{select}}(x)
&=
\int_{\mathcal Z}
U_{P_x^\star}(x,z)
\left[q_x^\star(z\mid x)-q_A(z\mid x)\right]\,dz,
\nonumber\\
\Delta_{\mathrm{real}}(x)
&=
\int_{\mathcal Z}
q_A(z\mid x)
\left[U_{P_x^\star}(x,z)-U_{P_\theta}(x,z)\right]\,dz.
\label{eq:gap-integrals}
\end{align}
Substituting Equation~\eqref{eq:supp-up} recovers the double-integral form in Equation (6) of the main paper.

\subsection{Persona States, Observable Behaviors, and Training Rewards}

A Persona state $z$ is not directly observable. The main paper therefore defines
\begin{equation}
a=g(\tau)
\label{eq:observable-behavior}
\end{equation}
as the observable high-level behavior extracted from a complete trajectory. Define the probability of producing behavior $a$ from Persona state $z$ as
\begin{equation}
G_\theta(a\mid x,z)
=
\int_{\mathcal T}
\mathbf 1[g(\tau)=a]P_\theta(\tau\mid x,z)\,d\tau.
\label{eq:behavior-kernel}
\end{equation}
Also define the observable behavior distribution under external Persona condition $p$ as
\begin{equation}
\rho_\theta(a\mid x,p)
=
\Pr_{\tau\sim\pi_\theta(\cdot\mid x,p)}
[g(\tau)=a].
\label{eq:behavior-distribution}
\end{equation}

\begin{proposition}[Observable behavior distribution]
The observable behavior distribution satisfies
\begin{equation}
\rho_\theta(a\mid x,p)
=
\int_{\mathcal Z}
q_\theta(z\mid x,p)G_\theta(a\mid x,z)\,dz.
\label{eq:observable-mixture}
\end{equation}
\end{proposition}

\begin{proof}
Using Equation~\eqref{eq:supp-decomposition} and exchanging the order of integration,
\begin{align}
\rho_\theta(a\mid x,p)
&=
\int_{\mathcal T}
\mathbf 1[g(\tau)=a]\pi_\theta(\tau\mid x,p)\,d\tau
\nonumber\\
&=
\int_{\mathcal Z}
q_\theta(z\mid x,p)
\left[
\vphantom{\int_{\mathcal T}}
\right.
\nonumber\\[-1mm]
&\qquad\left.
\int_{\mathcal T}
\mathbf 1[g(\tau)=a]P_\theta(\tau\mid x,z)\,d\tau
\right]dz
\nonumber\\
&=
\int_{\mathcal Z}
q_\theta(z\mid x,p)G_\theta(a\mid x,z)\,dz.
\label{eq:observable-proof}
\end{align}
\end{proof}

Thus, $a$ provides an observable proxy for the outcome of Persona Selection but is not identical to the Persona state $z$. Its distribution is jointly determined by Persona Selection and Persona Realization.

For a Runtime Alignment task $x$, define
\begin{equation}
S_x
=
\Pr_{\tau\sim\pi_\theta(\cdot\mid x,p_0)}
[g(\tau)=a^\star(x)]
\label{eq:sx}
\end{equation}
as the probability of producing the target high-level behavior. When $S_x>0$, define
\begin{equation}
H_x
=
\mathbb E_{\tau\sim\pi_\theta(\cdot\mid x,p_0)}
\!\left[
h_x(\tau\mid a^\star(x))
\mid g(\tau)=a^\star(x)
\right].
\label{eq:hx}
\end{equation}
When $S_x=0$, let $H_x=0$.

\begin{proposition}[Expected Runtime Alignment rewards]
The rewards defined in the main paper satisfy
\begin{align}
\mathbb E[r_S(x,\tau)]&=S_x,
\label{eq:ers}\\
\mathbb E[r_R(x,\tau)]&=S_xH_x,
\label{eq:err}
\end{align}
and
\begin{equation}
\mathbb E[R_x(\tau)]
=
S_x(\lambda_S+\lambda_RH_x).
\label{eq:er}
\end{equation}
\end{proposition}

\begin{proof}
Since $r_S$ is the indicator of the event $g(\tau)=a^\star(x)$,
\begin{equation}
\mathbb E[r_S(x,\tau)]=S_x.
\label{eq:ers-proof}
\end{equation}
By the definition of conditional expectation,
\begin{align}
\mathbb E[r_R(x,\tau)]
&=
\mathbb E[
\mathbf 1[g(\tau)=a^\star(x)]
h_x(\tau\mid a^\star(x))]
\nonumber\\
&=S_xH_x.
\label{eq:err-proof}
\end{align}
Finally, using
\begin{equation}
R_x(\tau)
=
\lambda_Sr_S(x,\tau)+\lambda_Rr_R(x,\tau),
\label{eq:reward-definition}
\end{equation}
and linearity of expectation,
\begin{align}
\mathbb E[R_x(\tau)]
&=\lambda_SS_x+\lambda_RS_xH_x
\nonumber\\
&=S_x(\lambda_S+\lambda_RH_x).
\label{eq:er-proof}
\end{align}
\end{proof}

Therefore, $r_S$ is an observable proxy for whether the target high-level behavior is expressed, while $r_R$ measures the joint outcome of expressing that behavior and realizing it effectively. Neither reward is a direct numerical estimate of the theoretical Selection Gap or Realization Gap.

\section{Training Data Construction}

The training data for R$^2$A are organized around its two-stage objective. Persona Representation Learning uses Who--How--What presentations to establish the Accountable-Professional Persona's overarching objective, conditional behavioral principles, and concrete manifestations in complete task trajectories. Persona Runtime Alignment then removes the explicit Persona specification and uses task states from mathematics, code, and grounded question answering to calibrate how these principles are selected and realized in practice. The data are therefore designed not to uniformly increase questioning, abstention, critique, or persistence, but to represent the conditions under which each behavior---and its complementary behavior---is appropriate.

\subsection{Who--How--What Data}

The supervised data for Persona Representation Learning consist of three presentation types, summarized in Table~\ref{tab:presentation-types}.

\begin{table}[t]
\centering
\small
\begin{tabular}{lp{0.69\columnwidth}}
\toprule
Presentation & Role \\
\midrule
Who & Defines the Persona's identity, responsibility, and overarching objective \\
How & Describes the applicability conditions and boundaries of the four principles \\
What & Demonstrates how these objectives and principles unfold in complete task trajectories \\
\bottomrule
\end{tabular}
\caption{The three presentation types used for Persona Representation Learning.}
\label{tab:presentation-types}
\end{table}

\paragraph{Who.}
The Who presentation defines the overarching objective of the Accountable-Professional Persona. Its purpose is not to induce a particular linguistic style, but to establish responsibility for the factual basis of a task, the quality of professional judgment, the next actionable step, and the eventual outcome. The examples also distinguish accountability from both excessive caution and unsupported completion: the model should neither stop feasible work merely to appear careful nor guess or fabricate merely to appear capable.

A representative training example is:

\begin{quote}
\small
\textbf{Representative Who Training Example}

\textbf{User:} What does being an accountable professional mean when you prepare decisions for an executive rather than making the final choice yourself?

\textbf{Assistant:} I treat accountability as ownership of the decision's factual foundation and usable next step. I surface conflicts, verify consequential details, distinguish evidence from preference, and make assumptions visible without burying the executive in avoidable questions. Once the choice is made, I preserve the rationale, execute the authorized actions carefully, and raise new evidence promptly if it changes the risk.
\end{quote}

This example does not teach a specific task skill. Instead, it specifies the objective that should organize subsequent judgments and actions.

\paragraph{How.}
The How presentation expresses the four Accountable-Professional principles as conditional boundaries between complementary behaviors:
\[
\begin{gathered}
\mathrm{ASK}/\mathrm{ANSWER}
\qquad
\mathrm{ABSTAIN}/\mathrm{PROCEED}
\\
\mathrm{CRITIQUE}/\mathrm{SOLVE}
\qquad
\mathrm{MAINTAIN}/\mathrm{UPDATE}.
\end{gathered}
\]
The examples do not independently reinforce one side of each pair. They specify when a behavior is appropriate and when its complementary behavior should instead be expressed. The model should ask only when missing information materially affects task completion, proceed when a task remains reliably feasible, avoid manufacturing objections when the premises are sound, and update only when stronger evidence changes the conclusion.

A representative training example is:

\begin{quote}
\small
\textbf{Representative How Training Example}

\textbf{User:} As the lead researcher, how should you handle a request for a definitive limitation-period analysis when the governing jurisdiction has not been identified?

\textbf{Assistant:} I would first isolate the fact that controls the legal conclusion: Which jurisdiction governs the dispute? Limitation rules can change materially across jurisdictions, so choosing one silently would create false certainty. While that answer is pending, I can organize the known dates, claims, and possible tolling facts, but I would reserve the definitive conclusion until the governing law is supplied.
\end{quote}

The example illustrates both the behavior and its boundary: the model asks for the missing condition that determines the conclusion, while continuing work that does not depend on that condition.

\paragraph{What.}
The What presentation instantiates the Persona's objective and conditional principles in complete trajectories from mathematics, code, and grounded question answering. It shows how the model interprets the current task state, takes an appropriate course of action, and carries the task through to a complete response.

The What layer draws from the same main source families used for Persona Runtime Alignment: DeepMath-103K \citep{he2025deepmath}, ACECODE-87K \citep{zeng2025acecoder}, and Fresh-Wiki QA. We use Qwen3-32B \citep{yang2025qwen3} as the trajectory-authoring model. Given a source task and the relevant Accountable-Professional principles, it generates candidate trajectories that include task-state interpretation, the resulting action, and the final task response.

Candidate trajectories are subsequently reviewed for three properties: whether the expressed behavior is appropriate for the task state, whether the subsequent trajectory realizes that behavior effectively, and whether the final response satisfies the original task. Mathematics, code, and grounded-QA trajectories are checked against their corresponding task references, execution conditions, or source evidence. Only complete and task-consistent trajectories are retained.

Who and How examples are authored in multiple independent batches and reviewed under a shared rubric. Repetitive, slogan-like, AI-self-referential, or one-sided examples that collapse a principle into a fixed response tendency are removed. Model-visible inputs exclude benchmark names, reward descriptions, and target behavior labels, and supervision is applied only to the designated assistant outputs.

\subsection{Persona Runtime Alignment Data}

Persona Runtime Alignment tasks are primarily constructed from DeepMath-103K, ACECODE-87K, and Fresh-Wiki QA, providing source tasks in mathematics, code, and grounded question answering, respectively. Their roles are summarized in Table~\ref{tab:runtime-sources}.

\begin{table*}[t]
\centering
\small
\resizebox{\textwidth}{!}{%
\begin{tabular}{lll}
\toprule
Source & Domain & Role in Runtime Alignment \\
\midrule
DeepMath-103K & Mathematical reasoning & Supports task states involving missing conditions, feasibility, premise validity, and subsequent evidence \\
ACECODE-87K & Code & Supports task states involving incomplete requirements, implementation feasibility, invalid assumptions, and execution feedback \\
Fresh-Wiki QA & Grounded question answering & Supports answerable, clarification-seeking, and abstention-requiring evidence states \\
\bottomrule
\end{tabular}}
\caption{Primary sources and roles of the Persona Runtime Alignment data.}
\label{tab:runtime-sources}
\end{table*}

The source tasks are adapted around the four Accountable-Professional principles rather than used only as ordinary capability-training examples.

\paragraph{Information Sufficiency.}
Tasks distinguish between contexts in which task-critical information is missing and contexts in which the available information is sufficient. When the missing information can be supplied through interaction, the model must identify the specific gap and request it; when the task is sufficiently specified, it should complete the task directly.

\paragraph{Feasibility Honesty.}
Tasks distinguish between requests for which a reliable outcome is unavailable and tasks that remain feasible. The former require the model to recognize and communicate the limitation, whereas the latter require it to continue rather than abstain merely because the task is difficult or unfamiliar.

\paragraph{Premise Scrutiny.}
Tasks distinguish invalid, missing, or contradictory premises from sound ones. The model must identify consequential premise defects when present, while avoiding unnecessary critique when the task is well formed.

\paragraph{Evidence Responsiveness.}
Multi-turn tasks distinguish unsupported pressure from genuinely stronger evidence. The model should maintain a supported conclusion under repetition, authority cues, or ungrounded challenges, but revise it when new evidence changes the task state.

The same principle is instantiated across multiple domains, and the same domain contains different behavioral states. This reduces the extent to which the policy can associate a behavior with a topic alone. During Runtime Alignment, the policy receives only a neutral task instruction and the interaction history currently available to it. The Accountable-Professional specification and target high-level behavior are not included in the policy input. The policy must therefore infer the appropriate behavior from the task state itself.

The tasks are designed so that selecting an appropriate high-level behavior is not sufficient on its own. A clarification question must recover the information that is actually missing; identifying a defective premise must lead to a useful continuation; and revising a conclusion must remain consistent with the newly available evidence. The resulting task trajectories therefore provide feedback on both behavior selection and trajectory-level realization.

Only tasks with a clear task state, a valid interaction process, and reliable outcome feedback are retained. Model-visible context is separated from task-internal information used to maintain the environment and assess the completed trajectory. The final training package is additionally checked for format integrity, masking correctness, source lineage, and conflicts with the frozen evaluation data.

\subsection{Fresh-Wiki QA}

Fresh-Wiki QA is a Wikipedia-derived training source constructed for this work and is unrelated to the public benchmark FreshQA \citep{vu2024freshllms}, which is not used for training. It mainly supports Information Sufficiency and Feasibility Honesty by creating different evidence states for otherwise similar factual questions.

Candidate pages are retrieved from English Wikipedia categories corresponding to 2025--2026 releases and entities. This temporal filtering is intended to reduce confounding from answers that the model can retrieve directly from parametric memory; it does not guarantee that the underlying facts were absent from pretraining. Fresh-Wiki QA consists of source-grounded, single-hop factual questions rather than multi-hop questions.

Each source record contains a question, an answer and its aliases, an authentic supporting sentence, the surrounding context, and source metadata. The answer must occur explicitly in the supporting sentence, and the supporting sentence must be an authentic substring of the corresponding Wikipedia text. Candidates without a reliable answer--evidence relation are removed rather than repaired with synthetic evidence.

Candidate questions undergo closed-book and open-context screening. Questions that can already be answered consistently without the source material are treated as overly easy or potentially affected by memorization. Questions that remain unreliable even when the complete evidence is supplied are treated as insufficiently clear. Retained records are also checked for answer--evidence consistency, source-lineage duplication, and textual collisions with evaluation data and other training sources.

For each retained fact, we control the availability of decisive evidence to construct three task states, summarized in Table~\ref{tab:freshwiki-states}.

\begin{table}[t]
\centering
\small
\begin{tabular}{p{0.19\columnwidth}p{0.27\columnwidth}p{0.27\columnwidth}}
\toprule
Task state & Available evidence & Expected behavior \\
\midrule
Answerable & The decisive supporting evidence is visible & Answer directly from the provided material \\
Clarification & Necessary information is initially hidden but can be obtained through interaction & Ask a targeted question and answer after the information is provided \\
Abstention & The decisive evidence is unavailable and cannot be obtained & State that the answer cannot be determined reliably rather than guess \\
\bottomrule
\end{tabular}
\caption{Fresh-Wiki QA task states constructed by controlling evidence availability.}
\label{tab:freshwiki-states}
\end{table}

These states are created by programmatically retaining, hiding, removing, or releasing authentic source evidence; the underlying fact is not rewritten. The answer, aliases, and supporting sentence remain hidden from the policy and are used only to maintain the interaction and assess the completed task. Consequently, the appropriate behavior is determined by the evidence currently available to the policy rather than by the topic, surface form, or an explicit Persona prompt.

\section{Training Details}

The main experiments use Qwen3-8B \citep{yang2025qwen3} as the backbone and proceed through Persona Representation Learning followed by Persona Runtime Alignment. Section~B.1 applies the same two-stage procedure to Qwen3-14B. Both stages use full-parameter training without LoRA \citep{hu2022lora}.

Stage I is trained for 2 epochs. Stage II is initialized from the final Stage-I checkpoint, with the optimizer and learning-rate scheduler reinitialized, and is trained for 1 epoch. All training runs are conducted on a single node with 8 NVIDIA H200 GPUs using BF16 precision and a maximum sequence length of 4,096 tokens.

\begin{table}[t]
\centering
\small
\resizebox{\columnwidth}{!}{%
\begin{tabular}{lcc}
\toprule
Configuration & \shortstack{Persona Representation\\Learning} & \shortstack{Persona Runtime\\Alignment} \\
\midrule
Epochs & 2 & 1 \\
Learning rate & $1\times10^{-5}$ & $1\times10^{-6}$ \\
Optimizer & AdamW & AdamW \\
Weight decay & 0.1 & 0.01 \\
Scheduler & Cosine & Constant \\
Gradient clipping & 1.0 & 1.0 \\
Rollouts per task & -- & 8 \\
Sampling & -- & \shortstack{Temperature 0.7,\\top-$p$ 0.95} \\
KL coefficient & -- & 0.02 \\
Precision & BF16 & BF16 \\
Hardware & $8\times$H200 & $8\times$H200 \\
\bottomrule
\end{tabular}}
\caption{Training configurations for the two stages of R$^2$A.}
\label{tab:training}
\end{table}

Stage I uses fully sharded data parallelism (FSDP) with full sharding and gradient checkpointing. Stage II uses FSDP to manage the actor and reference models, with SGLang for rollout generation.

\section{Benchmark and Metric Details}

\subsection{Common Evaluation Setup}

We evaluate ten benchmark families. Because QuestBench and ClarifyMT each contribute two separately reported subsets, the evaluation contains twelve settings in total. The benchmark families are QuestBench \citep{li2025questbench}, ClarifyMT \citep{luo2025clarifymt}, AbstentionBench \citep{kirichenko2025abstentionbench}, XSTest \citep{rottger2024xstest}, OR-Bench \citep{cui2025orbench}, PCBench \citep{li2025pcbench}, MiP \citep{fan2025missing}, MathTrap \citep{zhao2024mathtrap}, SYCON-Bench \citep{hong2025sycon}, and FaithEval \citep{ming2025faitheval}. All benchmark results are reported separately, and no cross-benchmark average is computed. The capabilities evaluated by each benchmark and their correspondence to the behavioral principles of the Accountable-Professional Persona are shown in Figure~1 of the main paper.

Within each model comparison, all conditions use the same evaluation instances, decoding configuration, and scoring procedure. Inference uses temperature 0, top-$p$ 1, one completion per instance, and random seed 31415.

We use deterministic parsers or matchers when applicable. Benchmarks requiring semantic assessment are evaluated using a frozen label-only Judge. Within each comparison, all model conditions share the same Judge configuration, scoring prompt, and label definitions.

\subsection{Benchmark Configurations}

The primary $n$ in Table~\ref{tab:benchmark-configs} denotes the denominator of the reported metric. Parenthetical totals additionally include paired controls or dialogue turns that are not necessarily included in the corresponding primary-metric denominator.

\begin{table*}[t]
\centering
\small
\resizebox{\textwidth}{!}{%
\begin{tabular}{llll}
\toprule
Benchmark and subset & Primary $n$ & Scoring method & Reported metric \\
\midrule
AbstentionBench-fast & 1,900 & Label-only Judge & Score-group mean abstention F1 $\uparrow$ \\
QuestBench-GSM-Q & 439 & Released parser & GSM-Q accuracy $\uparrow$ \\
QuestBench-Logic-Q & 1,150 & Released matcher & Logic-Q accuracy $\uparrow$ \\
ClarifyMT-Answer & 2,040 & Deterministic leading-action parser & Answer-side decision accuracy $\uparrow$ \\
ClarifyMT-Clarify & 4,080 & Deterministic leading-action parser & Clarify-side decision accuracy $\uparrow$ \\
PCBench-flawed/no-hint & 1,200 & Label-only Judge & Active Recognition Rate $\uparrow$ \\
MiP-missing & 984 (1,618 total) & Label-only Judge & Missing-premise detection rate $\uparrow$ \\
MathTrap-trap & 104 (208 total) & Label-only Judge & Trap accuracy $\uparrow$ \\
SYCON-Debate & 100 topics (500 turns) & Label-only Judge & Mean Turn of Flip $\uparrow$ \\
FaithEval-counterfactual & 1,000 & Normalized reference matcher & Counterfactual accuracy $\uparrow$ \\
XSTest-safe & 250 safe prompts (450 total) & Label-only Judge & Safe over-refusal rate $\downarrow$ \\
OR-Bench-Hard & 1,319 & Label-only Judge & Rejection rate $\downarrow$ \\
\bottomrule
\end{tabular}}
\caption{Benchmark configurations and reported metrics.}
\label{tab:benchmark-configs}
\end{table*}

All metrics are reported as percentages except Mean Turn of Flip. Arrows indicate the direction of better performance.

\subsection{Metric Definitions}

\paragraph{AbstentionBench-fast.}
We treat abstain as the positive class and compute binary F1 separately within each score group:
\begin{equation}
\mathrm{F1}_g
=
\frac{2\mathrm{TP}_g}
{2\mathrm{TP}_g+\mathrm{FP}_g+\mathrm{FN}_g},
\label{eq:abstention-f1-group}
\end{equation}
where $\mathrm{TP}_g$, $\mathrm{FP}_g$, and $\mathrm{FN}_g$ are defined with respect to abstention in score group $g$. The reported score is the equally weighted mean across the 19 score groups:
\begin{equation}
\mathrm{AbstentionF1}
=
\frac{1}{19}\sum_{g=1}^{19}\mathrm{F1}_g.
\label{eq:abstention-f1}
\end{equation}
The metric is therefore a score-group macro average rather than a single F1 computed after pooling all 1,900 instances. Higher values are better.

\paragraph{QuestBench-GSM-Q.}
The released parser extracts the option selected by the model. GSM-Q accuracy is defined as
\begin{equation}
\mathrm{GSMAcc}
=
\frac{N_{\mathrm{correct\ option}}}{439}.
\label{eq:gsmacc}
\end{equation}
Higher values indicate more accurate identification of the information required to complete the reasoning problem.

\paragraph{QuestBench-Logic-Q.}
The released matcher normalizes the model response and determines whether it matches the target logical attribute:
\begin{equation}
\mathrm{LogicAcc}
=
\frac{N_{\mathrm{normalized\ matches}}}{1{,}150}.
\label{eq:logicacc}
\end{equation}
Higher values indicate more accurate identification of the missing logical information.

\paragraph{ClarifyMT-Answer.}
For Answer-side instances, the appropriate initial action is to answer directly because the available information is sufficient. We parse the leading action of each response and compute
\begin{equation}
\mathrm{AnswerDecisionAcc}
=
\frac{N_{\mathrm{leading\ action=answer}}}{2{,}040}.
\label{eq:answerdecision}
\end{equation}
This metric evaluates whether the model chooses to answer. It does not evaluate the correctness of the subsequent answer content.

\paragraph{ClarifyMT-Clarify.}
For Clarify-side instances, the appropriate initial action is to request additional information. Clarify-side decision accuracy is
\begin{equation}
\mathrm{ClarifyDecisionAcc}
=
\frac{N_{\mathrm{leading\ action=clarify}}}{4{,}080}.
\label{eq:clarifydecision}
\end{equation}
This metric evaluates whether the model chooses to clarify. It does not evaluate the quality of the generated clarification question.

\paragraph{PCBench-flawed/no-hint.}
Active Recognition Rate measures whether the model proactively identifies and challenges the relevant false premise without being explicitly instructed to inspect it:
\begin{equation}
\mathrm{ARR}
=
\frac{N_{\mathrm{recognized}}}{1{,}200}.
\label{eq:arr}
\end{equation}
Here, $N_{\mathrm{recognized}}$ is the number of responses classified as identifying the relevant premise error. Higher values are better.

\paragraph{MiP-missing.}
Missing-premise detection rate measures whether the model recognizes that a necessary condition is absent:
\begin{equation}
\mathrm{MiPDetection}
=
\frac{N_{\mathrm{labeled\ insufficient}}}{984}.
\label{eq:mip}
\end{equation}
The primary metric is computed over the 984 missing-premise instances. The corresponding 634 original instances are evaluated as paired controls and are not included in the primary denominator.

\paragraph{MathTrap-trap.}
Trap accuracy measures whether the response correctly identifies the annotated contradiction, missing condition, undefined condition, or other defect that prevents the problem from being solved as stated:
\begin{equation}
\mathrm{TrapAcc}
=
\frac{N_{\mathrm{correctly\ identified}}}{104}.
\label{eq:trap}
\end{equation}
The primary metric is computed over 104 trap variants. Their 104 original counterparts are used as paired controls and are not included in the Trap accuracy denominator.

\paragraph{SYCON-Debate.}
Each topic contains five rounds of consecutive opposing feedback. For topic $i$, let
\begin{equation}
T_i\in\{0,1,\ldots,5\}
\label{eq:ti}
\end{equation}
denote the number of consecutive initial rounds for which the model remains aligned with its assigned stance. Counting stops when a response is first classified as neutral or opposing. Mean Turn of Flip is defined as
\begin{equation}
\mathrm{MeanToF}
=
\frac{1}{100}\sum_{i=1}^{100}T_i.
\label{eq:meantof}
\end{equation}
The metric ranges from 0 to 5. Higher values indicate that the model is more likely to maintain its stance under unsupported opposing feedback.

\paragraph{FaithEval-counterfactual.}
Counterfactual accuracy measures whether the model answers according to the supplied counterfactual context rather than defaulting to its prior world knowledge. After normalizing the response and reference answers, we compute
\begin{equation}
\mathrm{CounterfactualAcc}
=
\frac{N_{\mathrm{reference\ matches}}}{1{,}000}.
\label{eq:counterfactual}
\end{equation}
A response is counted as correct when it matches the normalized reference answer or an accepted alias supported by the supplied context. Higher values are better.

\paragraph{XSTest-safe.}
Safe over-refusal rate measures how often the model refuses requests that are safe and answerable:
\begin{equation}
\mathrm{SafeOverRefusal}
=
\frac{N_{\mathrm{partial\ refusal}}+N_{\mathrm{full\ refusal}}}{250}.
\label{eq:safe-overrefusal}
\end{equation}
The primary metric is computed over the 250 safe prompts. The remaining 200 unsafe prompts serve as controls and are not included in this denominator. Lower values are better.

\paragraph{OR-Bench-Hard.}
Rejection rate measures how often the model refuses benign requests from the Hard subset:
\begin{equation}
\mathrm{RejectionRate}
=
\frac{N_{\mathrm{indirect\ refusal}}+N_{\mathrm{direct\ refusal}}}{1{,}319}.
\label{eq:rejection}
\end{equation}
Both indirect and direct refusals are counted as rejections. Lower values are better.

\subsection{Direction-Aligned Visualization}

The result tables report all metrics in their original forms. Percentage-based metrics are reported as percentages, Mean Turn of Flip is reported on its original 0-to-5 scale, and the XSTest and OR-Bench metrics retain their lower-is-better directions.

For Figure~2 of the main paper, the metrics are converted to a common higher-is-better direction. Percentage values are first expressed as proportions. The three metrics requiring additional transformation are
\begin{align}
\mathrm{SYCON}_{\mathrm{display}}
&=\frac{\mathrm{MeanToF}}{5},
\label{eq:sycon-display}\\
\mathrm{XSTest}_{\mathrm{display}}
&=1-\mathrm{SafeOverRefusal},
\label{eq:xstest-display}\\
\mathrm{ORBench}_{\mathrm{display}}
&=1-\mathrm{RejectionRate}.
\label{eq:orbench-display}
\end{align}
These transformations are used only for visualization. The result tables report the original metrics, and the transformed values are not used to compute a cross-benchmark aggregate.

\subsection{Evaluation Scope}

For each semantic-scoring benchmark, the Judge configuration, scoring prompt, and label definitions are fixed across all model conditions within the same comparison.

The MathTrap primary result covers the 104 trap variants in the evaluated paired set, with the corresponding original problems used as controls. The OR-Bench result covers the 1,319 benign prompts in the Hard subset. The reported scores therefore refer specifically to these evaluated subsets.

\section{Construction of the Static Persona Prompt Families}

The Behavioral Frontier analysis in the main paper uses 165 static Persona conditions, consisting of 164 controlled prompt variants and the canonical Accountable-Professional prompt used in the main experiments. The canonical prompt is provided in Section~A.1. The 164 variants comprise two prompt families, Static64 and Wide100. All variants retain the same Accountable-Professional role description and other fixed instructions, while varying only the formulations of Information Sufficiency, Feasibility Honesty, and Premise Scrutiny---the three principles examined in the Behavioral Frontier analysis.

\subsection{Static64}

Static64 assigns four intervention levels to each of the three principles and combines them using a full-factorial design:
\begin{equation}
4_Q\times4_S\times4_C=64,
\label{eq:static64}
\end{equation}
where $Q$, $S$, and $C$ denote Information Sufficiency, Feasibility Honesty, and Premise Scrutiny, respectively. Each formulation specifies both a signature behavior and its complementary behavior. For example, the model should clarify when task-critical information is missing but answer directly when the available information is sufficient; state its limitations when a task cannot be completed reliably but proceed when it is feasible; and identify problematic premises while solving normally when the premises are sound. The levels therefore vary the applicability boundaries and intervention tendencies of each principle rather than simply adding or removing that principle.

\subsection{Wide100}

To cover a broader range of behaviors under static elicitation, Wide100 assigns five predefined formulations to each of $Q$, $S$, and $C$. These formulations range from completion-first and delayed intervention to balanced, broader, and more conservative intervention tendencies. The complete factorial design contains
\begin{equation}
5_Q\times5_S\times5_C=125
\label{eq:wide125}
\end{equation}
combinations. We retain 100 combinations using a fixed fractional-factorial design. Let
\begin{equation}
q,s,c\in\{0,\ldots,4\}
\label{eq:indices}
\end{equation}
denote the formulation indices for the three principles. Combinations satisfying
\begin{equation}
c=\bigl(3(q+s)+1\bigr)\bmod 5
\label{eq:excluded}
\end{equation}
are excluded. In the resulting design, each formulation appears 20 times along each dimension, and every pairwise combination of formulations appears four times.

Together, the 64 Static64 conditions, the 100 Wide100 conditions, and the canonical Accountable-Professional prompt form the 165 static Persona conditions used in the main paper. In the frontier figure, the 164 controlled variants are grouped as prompt points and the canonical prompt is marked separately. Qwen3-8B and R$^2$A are included only as external reference points and are excluded from the empirical Pareto-frontier construction.

\bibliography{recovered_references}

\end{document}